%% file: paper.tex
\documentclass{article}

\usepackage{microtype}
\usepackage{graphicx}
\usepackage{subfigure}
\usepackage{booktabs} 

\usepackage{hyperref}

\usepackage[accepted]{icml2020}

\usepackage{formattings/hadianeh}
\usepackage{formattings/macros}

\usepackage{calrsfs}

\DeclareMathAlphabet{\pazocal}{OMS}{zplm}{m}{n}

\usepackage{xspace}
\newcommand{\dcq}[0]{\textbf{\textsf{\textsc{DCQ}}}\xspace}
\usepackage{amsmath}
\usepackage{pdflscape}
\usepackage{breqn}

\usepackage{amsthm}
\newtheorem{theorem}{Theorem}
\newtheorem{lemma}{Lemma}

\usepackage{tikz}
\newcommand{\ballnumber}[1]{\tikz[baseline=(myanchor.base)] \node[circle,fill=.,inner sep=1pt] (myanchor) {\color{-.}\bfseries\footnotesize #1};}

\icmltitlerunning{DCQ: Divide and Conquer for Quantization}

\begin{document}

\twocolumn[
\icmltitle{Divide and Conquer: Leveraging Intermediate Feature Representations for \\ Quantized Training of Neural Networks}



\icmlsetsymbol{equal}{*}

\begin{icmlauthorlist}
\icmlauthor{Ahmed T. Elthakeb}{to}
\icmlauthor{Prannoy Pilligundla}{goo}
\icmlauthor{Alex  Cloninger}{ed}
\icmlauthor{Hadi Esmaeilzadeh}{goo} \\
\textbf{A}lternative \textbf{C}omputing \textbf{T}echnologies ({\color[HTML]{0B6121}{\textbf{ACT}}}) Lab\\
University of California San Diego  \\
{\sffamily\small {\{\href{mailto:a1yousse@eng.ucsd.edu}{a1yousse},\ \href{mailto:ppilligu@eng.ucsd.edu}{ppilligu}\}@eng.ucsd.edu} \quad \quad {\href{mailto:acloninger@ucsd.edu}{acloninger@ucsd.edu}} \quad \quad {\href{mailto:hadi@eng.ucsd.edu}{hadi}@eng.ucsd.edu}}
\end{icmlauthorlist}

\icmlaffiliation{to}{Department of Electrical and Computer Engineering, University of California San Diego.} 
\icmlaffiliation{goo}{Department of Computer Science, University of California San Diego.} 
\icmlaffiliation{ed}{Department of Mathematics, University of California San Diego.}

\icmlcorrespondingauthor{Ahmed T. Elthakeb}{a1yousse@eng.ucsd.edu}

\icmlkeywords{Machine Learning, ICML}

\vskip 0.3in
]



\printAffiliationsAndNotice{}  

\input{body/abstract}

\input{body/intro}
\input{body/method}

\input{body/evaluation}
\input{body/theory_backup}

\input{body/related}

\input{body/conclusion}

\bibliography{reference/paper}
\bibliographystyle{icml2020}


\clearpage

\appendix


\end{document}

%% file: body/abstract.tex
\begin{abstract}

The deep layers of modern neural networks extract a rather rich set of features as an input propagates through the network.
This paper sets out to harvest these rich intermediate representations for quantization with minimal accuracy loss while significantly reducing the memory footprint and compute intensity of the DNN.
This paper utilizes knowledge distillation through teacher-student paradigm~\cite{DBLP:journals/corr/HintonVD15} in a novel setting that exploits the feature extraction capability of DNNs for higher-accuracy quantization.
As such, our algorithm logically divides a pretrained full-precision DNN to multiple sections, each of which exposes intermediate features to train a team of students independently in the quantized domain.
This divide and conquer strategy, in fact, makes the training of each student section possible in isolation while all these independently trained sections are later stitched together to form the equivalent fully quantized network.
Our algorithm is a sectional approach towards knowledge distillation and is not treating the intermediate representation as a hint for pretraining before one knowledge distillation pass over the entire network~\cite{DBLP:journals/corr/RomeroBKCGB14}.
%
%
Experiments on various DNNs (AlexNet, LeNet, MobileNet, ResNet-18, ResNet-20, SVHN and VGG-11) show that, this approach---called {\dcq} (\textbf{D}ivide and \textbf{C}onquer \textbf{Q}uantization)---on average, improves the performance of a state-of-the-art quantized training technique, DoReFa-Net~\cite{Zhou2016DoReFaNetTL} by 21.6\% and 9.3\% for binary and ternary quantization, respectively.
%
Additionally, we show that incorporating \dcq to existing quantized training methods leads to improved accuracies as compared to previously reported by multiple state-of-the-art quantized training methods.


\end{abstract}

%% file: body/intro.tex
\section{Introduction}
\label{sec:intro}
%
Today deep learning, with its superior performance, dominates a wide range of real life inference tasks including image recognition, voice assistants, and natural language processing~\cite{sirius,DBLP:conf/nips/KrizhevskySH12,DBLP:journals/nature/LeCunBH15,bpzip}. 
However, the shear complexity of deep learning models and the associated heavy compute and memory requirement appears as a major challenge as the demand for such services rapidly scale.
Quantization, which can reduce the complexity of each operation as well as the overall storage requirements of the DNN, has proven to be a promising path forward.
Nevertheless, quantization requires carefully tailored training and recovery algorithms~\cite{DBLP:conf/nips/CourbariauxBD15,DBLP:conf/icml/GuptaAGN15,DBLP:journals/jmlr/HubaraCSEB17,DBLP:conf/iclr/ZhouYGXC17,Zhou2016DoReFaNetTL} to even partially overcome its losses in accuracy.
\begin{figure}
    \centering
    \includegraphics[width=0.5\textwidth]{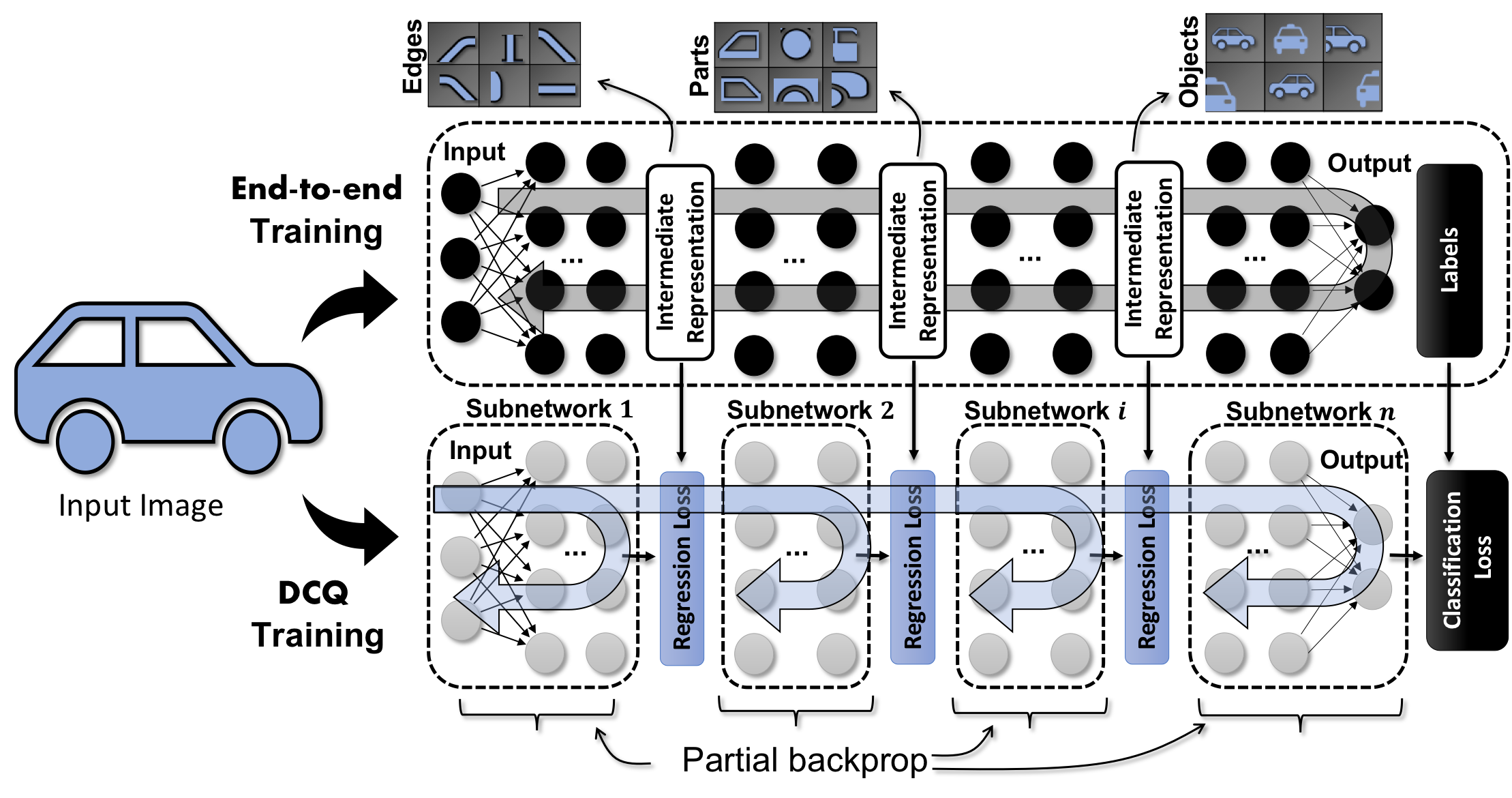}
    \vspace{-2ex}
   \caption{Overview of Divide and Conquer Quantization.}		  
    \vspace{-2ex}
    \label{fig:overview}
\end{figure}
In this paper, we set out to devise an algorithm that enables quantization with much less accuracy degradation.
The key insight is that the intermediate layers of a deep network already extract a very rich set of features and these intermediate representations can be used to train/teach a quantized network more effectively.
To that end, we define a new approach towards knowledge distillation through teacher-student paradigm~\cite{DBLP:journals/corr/HintonVD15,DBLP:conf/kdd/BucilaCN06} focusing on teaching the knowledge of intermediate features to a corresponding quantized student.
Knowledge distillation~\cite{DBLP:journals/corr/HintonVD15} is a generic approach to reduce a large model down to a simpler or smaller distilled model. 
At a high level, a softened version of the final output is used to train a small model (student) to mimic the behavior of the original large model (teacher).
\textsc{FitNets}~\cite{DBLP:journals/corr/RomeroBKCGB14} extends this idea and takes hints from an intermediate layer of the teacher to pretrain the first few layers of the student and then apply knowledge distillation to the entire student network.
We, on the other hand, tap into the multiple intermediate layers and apply knowledge distillation through sectioning.
The sectioning enables \dcq to train each section of the student independently in isolation to deliver a quantized counterpart for the teacher.
In fact, the hints as proposed in \textsc{FitNets} are complementary and can potentially be used in our sectional  knowledge distillation.
The proposed algorithm, \dcq, employs a divide and conquer approach that divides a pretrained full-precision network into multiple sections, each of which exposes a set of intermediate features. 
As Figure~\ref{fig:overview} illustrates, \dcq allocates a student section to each teacher counterpart and independently trains them using the intermediate feature representations.
\dcq calculates the loss of each student section by comparing it with the output activations of the corresponding teacher section of the full precision network.
Loss is optimized through a sectional multi-backpropagation scheme using conventional gradient-based training as shown in Figure \ref{fig:overview}.
These trained student sections are then sewed back together to form the corresponding quantized DNN.

We validate our method, \dcq, through experiments on a variety of DNNs including AlexNet, LeNet, MobileNet, ResNet-18, ResNet-20, SVHN and VGG-11 with binary and ternary weights.
%
%
%
%
Results show that \dcq, on average, improves the performance of a state-of-the-art quantized training technique, DoReFa-Net~\cite{Zhou2016DoReFaNetTL} by 21.6\% and 9.3\% for binary and ternary quantization, respectively, which further helps in closing the accuracy gap between state-of-the-art quantized training techniques and the full-precision runs.
Additionally, we show that our approach, \dcq, can improve performance of existing knowledge-distillation based approaches~\cite{Mishra2017WRPNWR} and multiple state-of-the-art quantized training methods.
These encouraging results suggest that leveraging the inherent feature extraction ability of DNNs for knowledge distillation can lead to significant improvement in their efficiency, reducing their bitwidth in this particular case.

The contributions of this paper can be summarized as follows.
\begin{itemize}[noitemsep]
  \vspace{-0.3cm}
  \item \textbf{Extending knowledge distillation.} \dcq enables leveraging \textit{arbitrary number of intermediate layers} relying on the inherent hierarchical learning characteristic of deep neural networks in contrast to only the output layer or hint layer. As such, distillation learning and hint learning fall as special cases of the proposed \textit{divide and conquer strategy}.

  \item \textbf{Enabling parallelization towards training quantized networks.} \dcq applies knowledge distillation through sectioning. As such it trains each section of the students independently in isolation to deliver a quantized counterpart for the teacher, which can occur in parallel. 
  \item \textbf{Complementary to other methods.} \dcq is a complementary method as it acts as an auxiliary approach to boost performance of existing training techniques by applying whatever the underlying training technique but in a stage-wise fashion with defining a regression loss per stage. 
  \item \textbf{Theoretical analysis.} We provide a theoretical analysis/guarantee of the error upper bound across the network through a chaining argument. 
\end{itemize}


%% file: body/method.tex
\section{DCQ: Divide and Conquer for Quantization}

\SetKwComment{Comment}{$\triangleright$\ }{}
\niparagraph{Overview.}
We take inspiration from knowledge distillation and apply it to the context of quantization by proposing a novel technique dubbed \dcq.
The main intuition behind \dcq is that a deeply quantized network can achieve accuracies similar to full precision networks if intermediate layers of the quantized network can retain the intermediate feature representations that was learnt by the full precision network.
To this end, \dcq splits the quantized network and full precision network into multiple small sections and trains each section individually by means of partial backpropagation so that every section of the quantized network learns and represents similar features as the corresponding section in the full precision network.
In other words, DCQ divides the original classification problem into multiple regression problems by matching the intermediate feature (activation) maps.
The following points summarizes the practical significance and contribution of \dcq.

\niparagraph{Weight and activation quantization.} The proposed technique is orthogonal to the quantity of interest for quantization, as it's basically applying whatever the underlying/used training technique but in a stage-wise fashion with defining a new regression loss per stage. In fact, the regression loss is defined to match the respective activation maps for each stage. As such, DCQ can be equally applied for weight and/or activation quantization alike. Section~\ref{sec:evaluation} presents results for both weight and activation quantization.

\niparagraph{Integration to other methods.} The proposed technique is a complementary method as it acts as an auxiliary approach to boost performance of existing training techniques by applying whatever the underlying/used training technique but in a stage-wise fashion with defining a new regression loss per stage.

%
\niparagraph{Knowledge distillation utilization.} \dcq extends the concept of knowledge distillation to its limits by leveraging multiple intermediate layers as opposed to limiting it to the output layer only as in~\cite{apprentice:2017}, ~\cite{DBLP:journals/corr/HintonVD15} or the output layer and hint layer as in~\cite{DBLP:journals/corr/RomeroBKCGB14}.

\niparagraph{Other performance benefits.} \dcq enables per-network training "parallelization" by enabling training different sections/stages in isolation (stage-wise fashion). Moreover, it applies the standard back propagation in a simpler settings (small subnetworks) which enables both faster convergence time and higher accuracy than existing conventional fine-tuning methods in the quantized domain.

This section describes different steps and rationale of our technique in more detail.
\subsection{Matching activations for intermediate layers}\label{sec:matching_activations}
\begin{figure}
    \centering
    \includegraphics[width=0.4\textwidth]{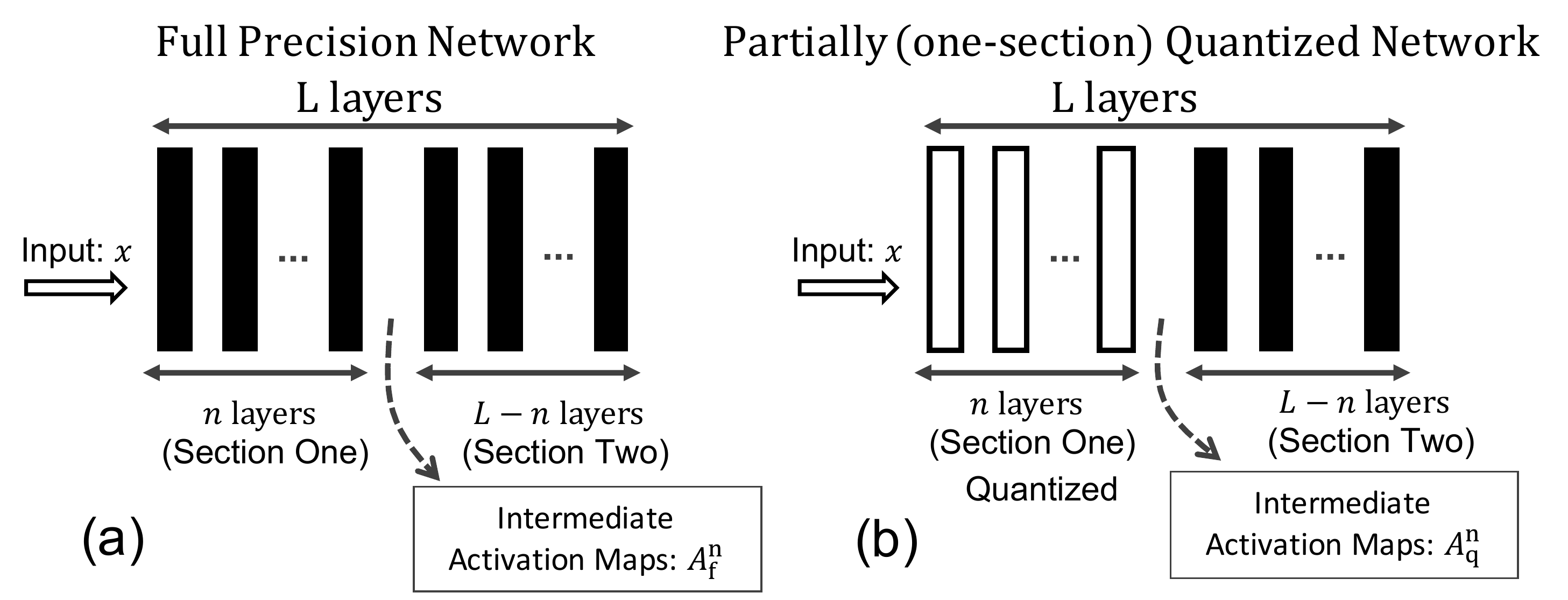}
        \vspace{-0.3cm}
    \caption{DCQ two stage split example}		  
    \vspace{-0.5cm}
    \label{fig:example}
\end{figure}
Figure~\ref{fig:example} (a) shows a sketch representing a full precision network of $L$ layers, whereas Figure~\ref{fig:example} (b) is a deeply quantized version of the same network where first $n$ layers are quantized and the remaining $L-n$ layers are at full precision.
When we pass the same input image $x$ to both these networks, if the output activations of layer $n$ for full precision network, i.e., $A_f^n$, are equivalent to the output activations of layer $n$ for the semi-quantized network, $A_q^n$, then both the networks classify the input to a same class because rest of the $L-n$ layers are same for both the networks and their input activations are same as well.
Therefore, if both these networks shown in Figure~\ref{fig:example} (a) and (b), have similar output activations for all the input images, then the network with first $n$ layers quantized has learnt to represent similar features as the first $n$ layers of the original network and it will have the same classification accuracy as the full precision network.
We can extend this argument further and say that if we quantize the remaining $L-n$ layers of the network in Figure~\ref{fig:example} (b) while keeping it's output same as the corresponding $L-n$ layers of the full precision network, then we now have a deeply quantized network with the same accuracy as the full precision network.
This is the underlying principle for our proposed quantization technique \dcq.
In the above example, the network was split into two sections of $n$ and $L-n$ layers, instead \dcq splits the original network into multiple sections and trains those sections individually to output same activations as the corresponding section in the full precision network.
Following subsections explain the \dcq methodology in more detail.
%

\subsection{Splitting, training and merging}\label{sec:split_train}
\begin{figure}
    \centering
    \includegraphics[width=0.5\textwidth]{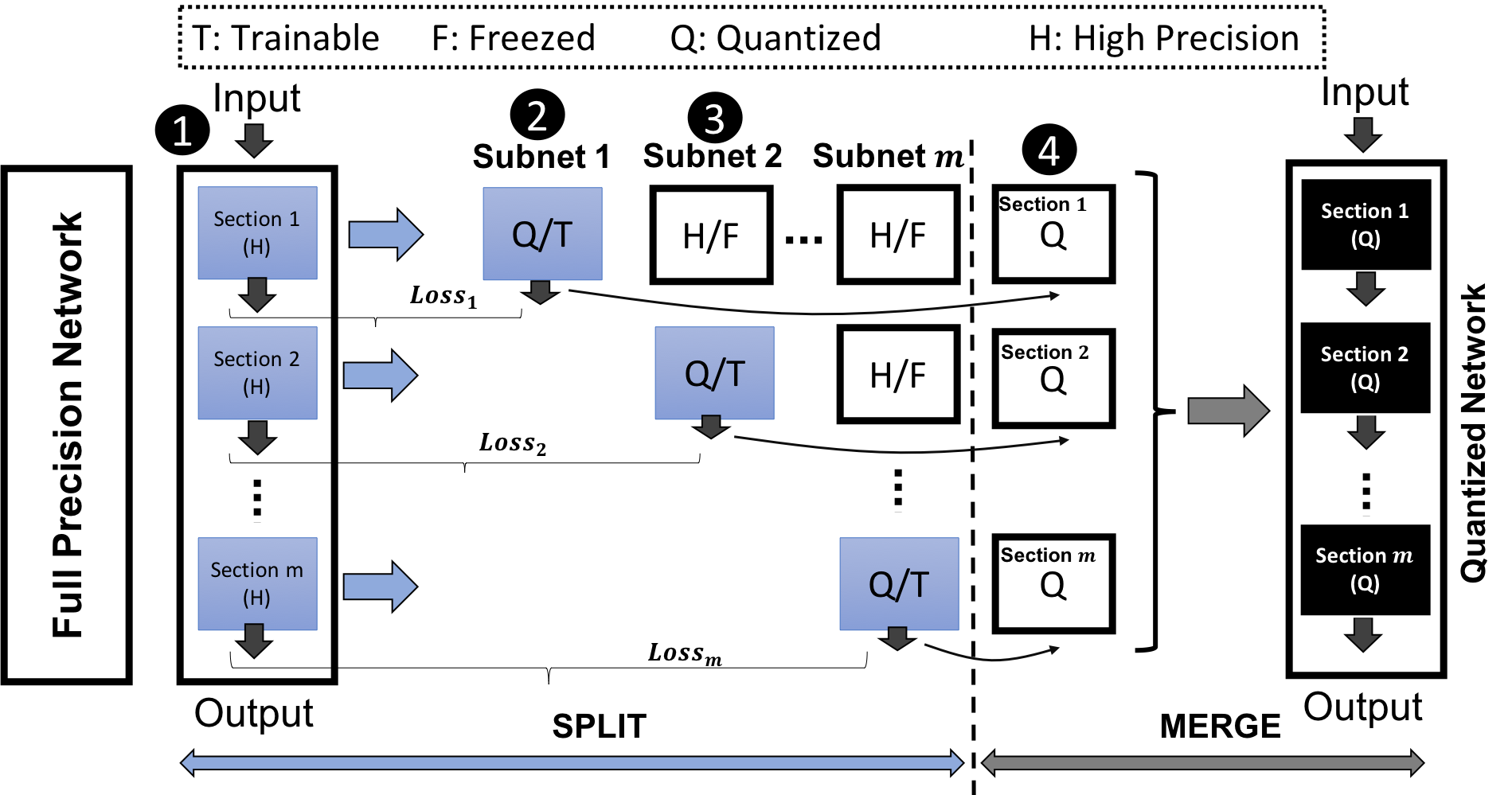}
    \caption{Divide and Conquer approach overview showing SPLIT phase; dividing the teacher full precision network into smaller subnetworks, and MERGE; by combining the training results of each subnetwork to form a fully quantized network}		  
    \vspace{-4ex}
    \label{fig:DCQ}
\end{figure}

\niparagraph{Splitting the full precision network.}
As described in Section~\ref{sec:matching_activations}, \dcq splits the original network into multiple sections and trains them in isolation and in parallel.
Figure~\ref{fig:DCQ} shows an overview of the entire process.
As shown in the figure, after splitting the full precision network into $m$ sub sections, \dcq quantizes and trains these subsections independently.
After training, \dcq puts them all together again to get the deeply quantized version of the entire original network.
As discussed in Section~\ref{sec:matching_activations}, because each of these sections is trained to capture the same features as the full precision network, although these sections are trained independently, they can be put together at the end to give similar accuracy as the full precision network.

If the original network has $L$ layers then $m$ decides how many layers will be part of each section (sections need not be equal in terms of number of layers).
In this work, we used a configuration of two layers per every section and then decided $m$ according to the total number of layers in the network.
Although, for networks like ResNet which have logical splits in terms of basic blocks, we split the network in a way that each section corresponds to a basic block.
We leave the task of deciding the optimal number of sections (splits) and how many layers per section for a given network to future work.
However, we provide some empirical analysis to this regard in Section~\ref{sec:exploratory_studies}.

\niparagraph{Training the sub-networks.}
As Figure~\ref{fig:DCQ} illustrates, \ballnumber{1} we create $m$ sections in order to train each of the $m$ sub-networks.
For each section $i$, the sub-network $i$ (or subnet $i$ for short) consists of all the sections preceding it.
Subnet 1 column in Figure~\ref{fig:DCQ} \ballnumber{2} shows a subnet for section $1$. 
To train this section, the output activations of the quantized version of section $1$ are compared with the output activations of the full precision version of section $1$ and the loss is calculated accordingly.
Section~\ref{sec:loss_func} gives more details on how the loss is calculated for each subnet.
Similarly, Subnet 2 column \ballnumber{3} shows the subnet for section $2$ and it comprises of both section $1$ and section $2$.
Output activations of section $2$ are used to calculate the loss in this case.
Since section $2$ is being trained in this subnet, weights for section $1$ are frozen(not trainable) in this subnet and backpropagation based on the loss only affects section $2$.
Similarly there are subnets for sections $3$ up to last section $m$ and the last subnet $m$ is basically similar to the full precision network except that the section $m$ is quantized and all the other sections from $1$ to $m-1$ are frozen.

\niparagraph{Merging the sections.}
\ballnumber{4} After training all the sections, since each of these sections has been trained independently to learn the same features as the corresponding section of the full precision network but with quantized weights, they can be put together to form a fully trained quantized network.
In every subnet, freezing all the sections except the one being trained is the key in enabling merging of all the individual sections at the end.

\subsection{Loss function for training sub networks}\label{sec:loss_func}
All machine learning algorithms rely on minimizing a loss function to achieve a certain objective.
The parameters of the network are trained by back-propagating the derivative of the loss with respect to the parameters throughout the network, and updating the parameters via stochastic gradient descent.
Broadly speaking, according to the particular task, loss functions can be categorized into two types: Classification Loss and Regression Loss.
Fundamentally, classification is about predicting a label (discrete valued output) and regression is about predicting a quantity (continuous valued output).
Since \dcq aims to capture the intermediate features learnt by the full precision network, loss needs to be calculated based on the output activations of intermediate layers unlike the traditional loss which is calculated using the output of the final classification layer and the targets.
As such, and in the context of this paper focusing on classification tasks, \dcq divides the original classification problem into multiple \textit{regression} problems by matching the intermediate feature (activation) maps.
%
%
%
In this study, we have examined three of the most commonly used regression loss formulations. 
Namely: \\
(1) Mean Square Error (MSE): $\mathcal{L}=\frac{1}{n}\sum_{i=1}^{n}(y^{(i)}-\hat{y}^{(i)})^2$, 
(2) Mean Absolute Error (MAE): $\mathcal{L}=\frac{1}{n}\sum_{i=1}^{n}\lvert y^{(i)}-\hat{y}^{(i)} \rvert$, 
and (3) Huber Loss: $$\mathcal{L}= \frac{1}{n}\sum_{i=1}^{n} \begin{cases} 
      \frac{1}{2}(y^{(i)}-\hat{y}^{(i)})^2 & , \ \lvert y^{(i)}-\hat{y}^{(i)} \rvert \leq \delta \\
      \delta (y^{(i)}-\hat{y}^{(i)}) - \frac{1}{2}\delta & , \ otherwise
   \end{cases}$$
%
%
where $y$ is the target value, and $\hat{y}$ is the predicted value, and the summation is across all samples.
For Huber loss, $\delta$ (delta) is a hyperparameter which can be tuned. Huber loss approaches MAE when $\delta \sim 0$ and MSE when $\delta \sim \infty$ (large numbers).
Section~\ref{sec:exploratory_studies} provides experimental results for each of the above loss formulations.
%
%
%



\subsection{Overall Algorithm}\label{sec:algorithm}
Algorithm~\ref{alg:overvall_technique} outlines the step by step procedure for \dcq putting together all the steps described in Sections~\ref{sec:split_train} and \ref{sec:loss_func}.
Since each iteration of the loop, shown in the algorithm, is independent, all the sections can potentially be trained in parallel leading to an overall reduction in training time.
\begin{algorithm}
\scriptsize
\caption{\scriptsize{Divide and Conquer for Quantization: Training Procedure}}
 \textbf{Input:} Pretrained Full Precision Neural Network  $(N^{FP})$  \\
 \textbf{Output:} Quantized Neural Network $(N^Q)$ 
\begin{algorithmic} [1]
\STATE {Split $N^{FP}$ into $\textbf{m}$ sections: $\{N_1, N_2, ..., N_m\}$ \Comment*[r]{\textbf{SPLIT phase}}
\STATE Each section $N_i$ has a set of layers: $\{l_1, l_2, ...,l_n\}$}
\FOR{$N_i$ in $\{N_1, N_2, ..., N_m\}$}  
\STATE Create a subnet $SN_i$ for section $N_i$ containing all the sections from $N_1$ to $N_i$ 
\STATE Quantize all the layers in section $N_i$ with the desired bitwidth to get $N^q_i$
\STATE Set all layers of section $N_i$ as trainable, freeze all other remaining layers in the subnetwork $SN_i$
\STATE Calculate $LOSS_i$ using the output activations of section $N_i$ of the full precision network and the subnetwork $SN_i$
\STATE  Minimize $\{LOSS_i\}$ to train $N^q_i$ to represent similar features as $N_i$
\ENDFOR
\STATE $N^Q \leftarrow$ merge$\{N^q_1, N^q_2, ..., N^q_m\}$ 
\Comment*[r]{\textbf{MERGE phase}}
\end{algorithmic}
\label{alg:overvall_technique}
\end{algorithm}
\vspace{-0.5cm}

%% file: body/evaluation.tex
\section{Experimental Results}
%
%
\begin{table}[t]
	\centering
	\caption{Summary of results comparing \dcq (our appraoch) to DoReFa-Net for different networks 
		considering binary and ternary weight quantization.}	
	\includegraphics[width=1\linewidth]{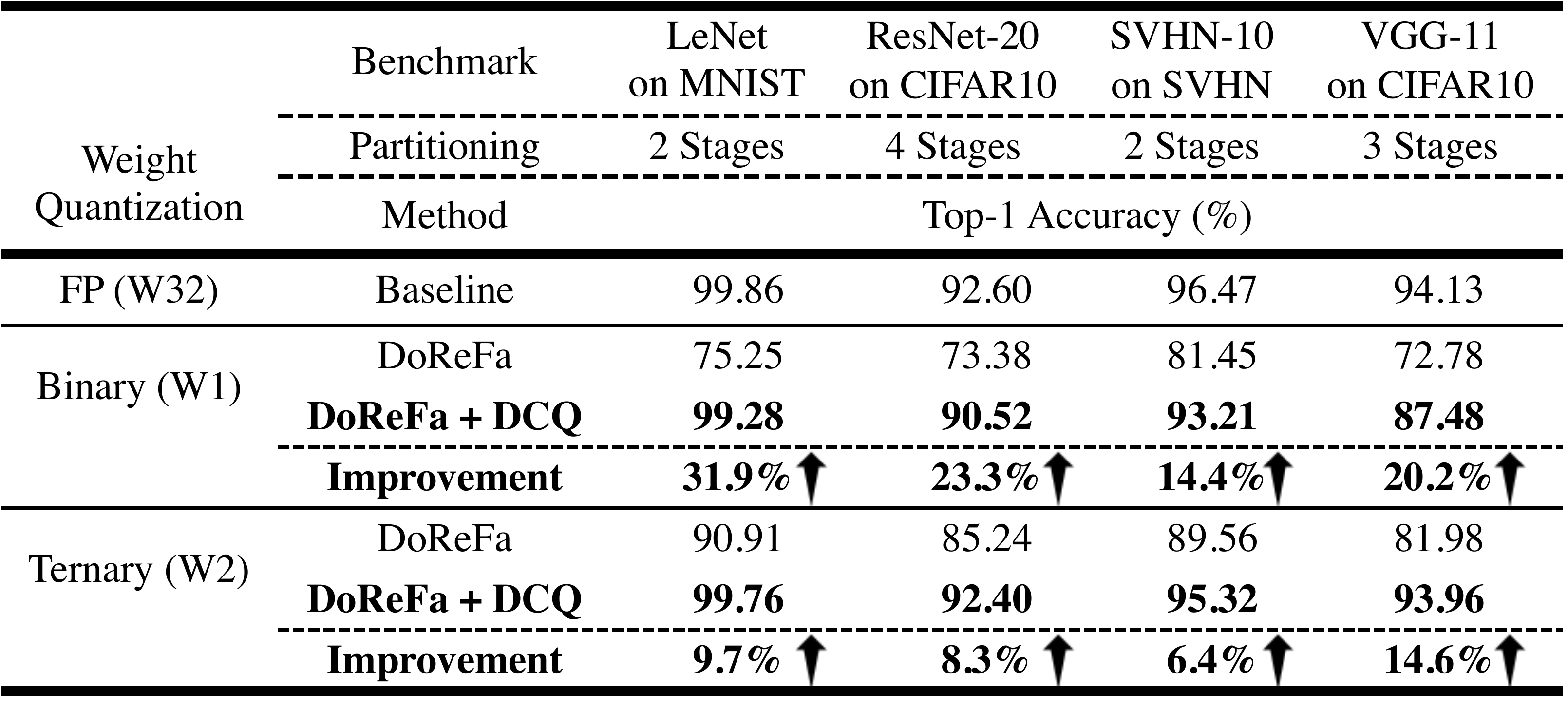}
	\label{table:t1}
	\vspace{-0.6cm}
\end{table}
\begin{table}[t]
	\centering
	\caption{Summary of results comparing our approach (DCQ) to state-of-the-art quantized training methods.}	
	\includegraphics[width=1\linewidth]{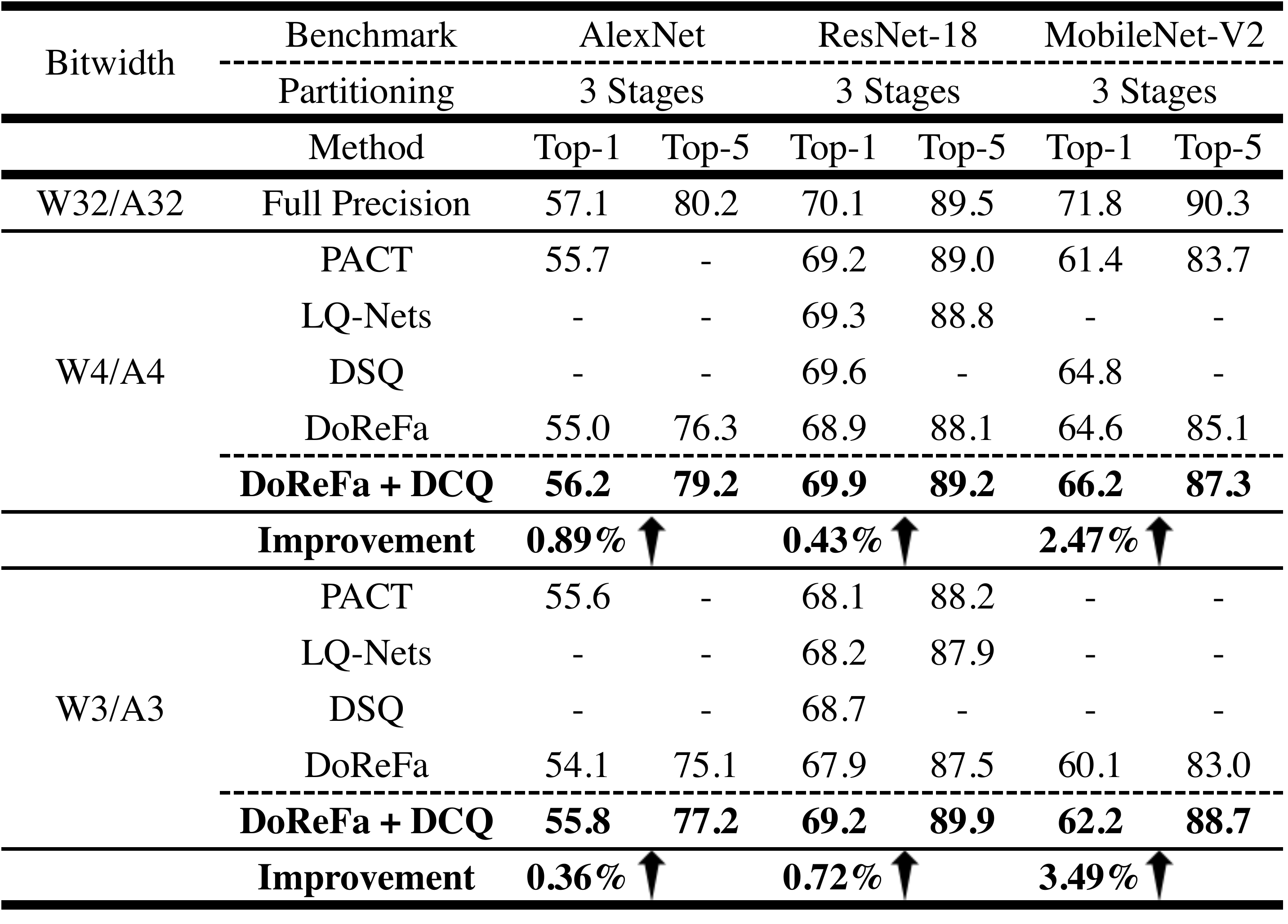}
	\label{table:t2}
	\vspace{-0.5cm}
\end{table}
\begin{table}
	\centering
	\caption{Comparing \dcq to a knowledge distillation based quantization method, Apprentice.}	
	\includegraphics[width=1\linewidth]{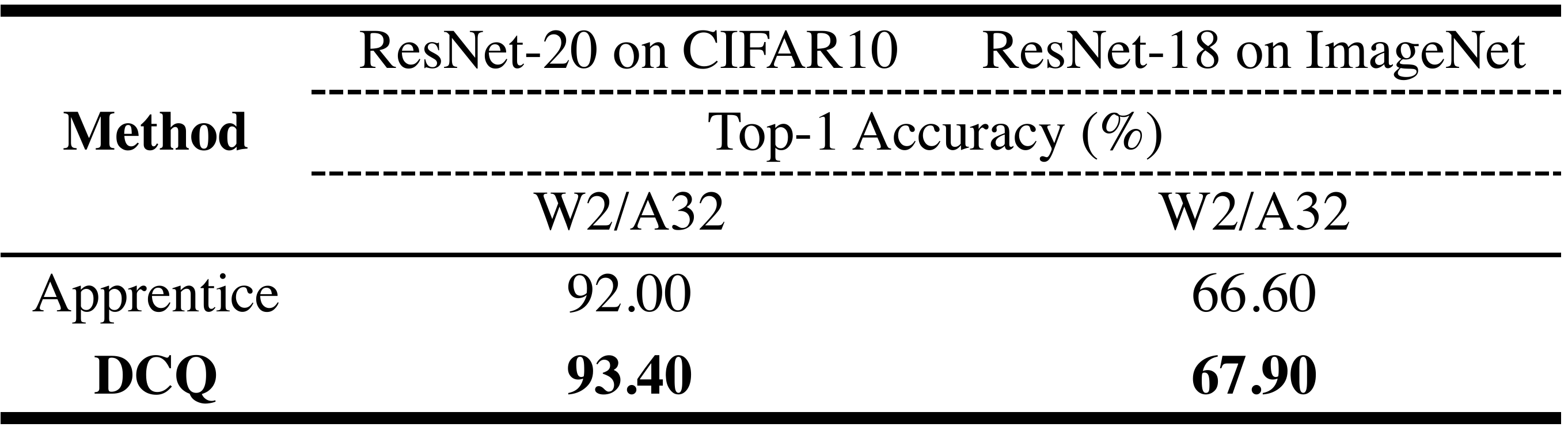}
	\label{table:t3}
	\vspace{-0.6cm}
\end{table}
%
%
%
\subsection{Experimental Setup} 
In this section, we evaluate the efficacy of our proposed approach on various DNNs (AlexNet, LeNet, MobileNet, ResNet-18, ResNet-20, SVHN and VGG-11) and different datasets: CIFAR10, ImageNet, MNIST, and SVHN.
We compare our approach to conventional end-to-end training approach. We consider DoReFa-Net~\cite{Zhou2016DoReFaNetTL} as our baseline but also show comparision with BWN~\cite{xnornet:eccv:2016} in Section~\ref{sec:xnor_comparision}, and Apprentice~\cite{apprentice:2017} , in addition to state-of-the-art quantized training methods: PACT~\cite{Choi2018PACTPC},  LQ-Net~\cite{DBLP:conf/eccv/ZhangYYH18}, DSQ~\cite{DBLP:journals/corr/abs-1908-05033} in Section~\ref{sec:apprentice_comparision}.

For all the experiments, we use an open source framework for quantization, Distiller~\cite{neta_zmora_2018_1297430}.
%
%
While reporting accuracies in their paper, DoReFa-Net doesn't quantize first and last layers of the network whereas in our case, we quantize all the layers including the first and last layers.
Because of this difference in quantization and using built-in implementation of Distiller, the accuracies we report might not exactly match the accuracies reported in their paper.
%
%
%
\subsection{ Binarization and Ternarization using DCQ}\label{sec:evaluation}
Table\ref{table:t1} shows summary of results comparing plain DoReFa to DoReFa + \dcq for different networks considering binary \{-1,1\} and ternary \{-1, 0, 1\} weight quantization for various networks: LeNet, ResNet-20, SVHN and VGG-11.
As seen, integrating \dcq into DoReFa outperforms the conventional approach and achieves a consistent improvements across the different networks with average 22.45\% for binarization and 9.7\% for ternarization.
Delving into the results, the reported improvements can be attributed to the following reasons.
First, deep multi-hidden-layer neural networks are much more difficult to tackle as compared to shallower ones. Furthermore, end-to-end backpropagation can be inefficient~\cite{DBLP:conf/icml/JaderbergCOVGSK17}. Thus, adopting such divide and conquer approach yields simpler subproblems that are easier to optimize.
Second, matching intermediate learning objectives also guides the optimization as compared to following a single global objective that indirectly specifies learning objectives to the intermediate layers.
%

\niparagraph{Comparison with BWN.}\label{sec:xnor_comparision}
BWN~\cite{xnornet:eccv:2016} proposes approximate convolutions using binary operations for a set of networks.
%
%
We show comparison on LeNet as it is the only common benchmark between both the works.
As Table~\ref{table:t1} shows, our technique achieves an accuracy of 99.3\%, which is close to the accuracy of 99.2\% reported by BWN.
%
However, BWN involves restructuring the original network architecture whereas our implementation does not introduce any changes to the architecture.
%

%
\subsection{Comparison with State-of-the-Art Quantized Training Methods}\label{sec:apprentice_comparision}
Here, we provide comparison to multiple state-of-the-art quantized training methods considering both weights and activation quantization.
Table~\ref{table:t2} summarizes the results of comparing to PACT, LQ-Net, DSQ, and DoReFa (the baseline) for several networks (AlexNet, ResNet-18, MobileNet).
As seen, \dcq outperforms these previously reported accuracies and achieves on average improvements of 0.98\%, and 0.96\% for W4/A4 and W3/A3, respectively. 

We also provide a comparison against knowledge distillation-based method Apprentice~\cite{apprentice:2017}, a recent work which also combines knowledge distillation with quantization.
Table~\ref{table:t3} shows that our technique outperforms Apprentice for both ResNet-20 on CIFAR10, and ResNet-18 on ImageNet considering ternary weights quantization.
%
%
The reported improvement can be attributed to the fact that \dcq combines the conventional knowledge distillation approach, as in~\cite{apprentice:2017}, in addition to its unique intermediate learning approach by regressing the quantized network intermediate feature maps to the corresponding full precision ones in a stage wise fashion. 
Moreover, the network architecture of the student network in~\cite{apprentice:2017} is typically different from that of the teacher network as opposed to \dcq where same network architecture is utilized for the student network but with quantized weights.
From one side, this saves a huge amount of effort designing a student network architecture which might incur significant hyperparameter tuning. 
On the other side, it enables a direct finetuning instead of a complete training from scratch as a result of preserving the original network architecture.
\begin{figure*} 
    \centering
    \includegraphics[width=0.8\textwidth]{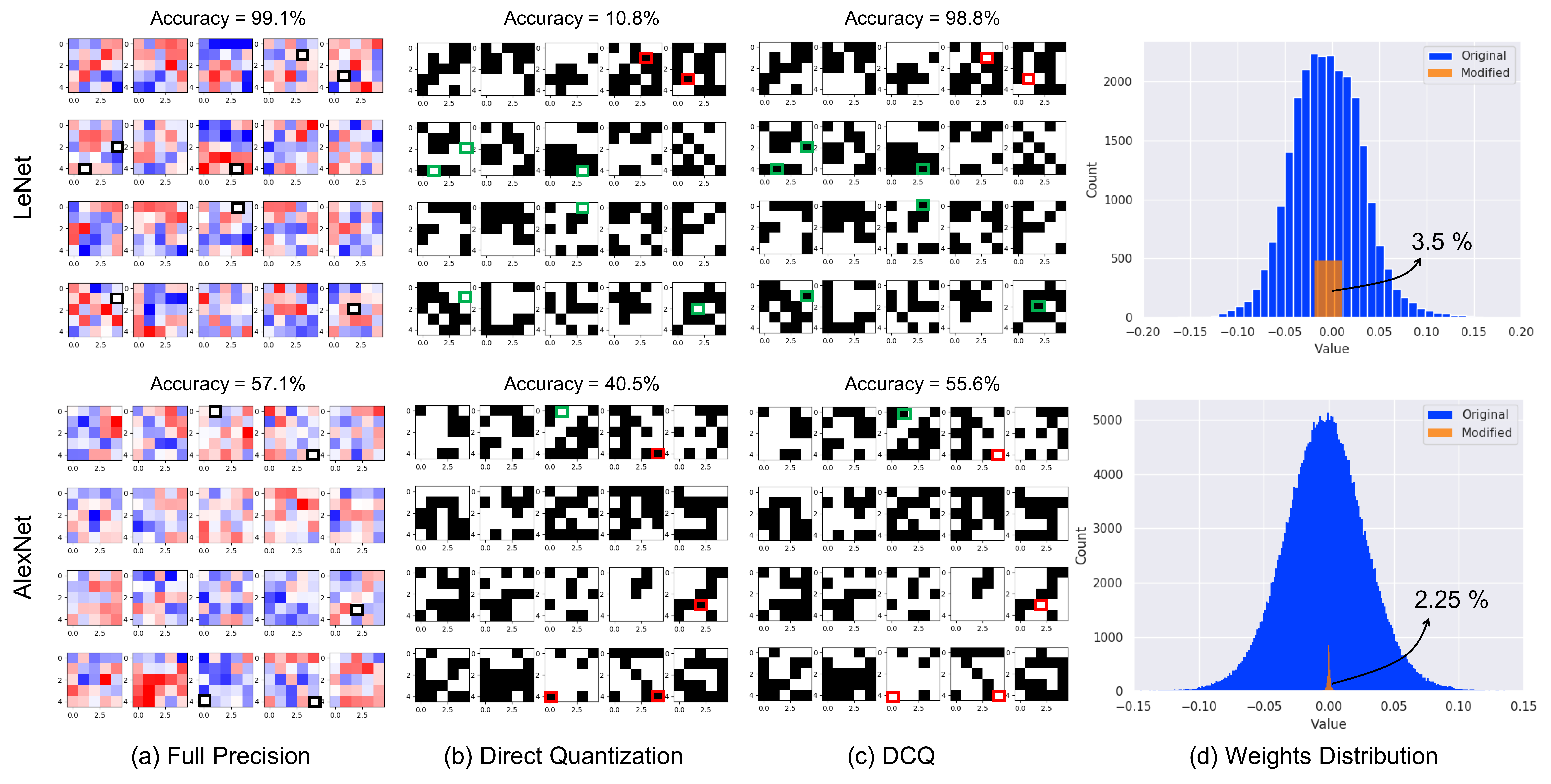}
    \caption{Visualization of a subset of weight kernels of the second convolutional layer of LeNet (top row), and AlexNet (bottom row), highlighting the differences between different versions of binary weight kernels: (a) Full precision weight kernels, (b) binary weight kernels upon direct binarization from full precision, (c) binary weight kernels obtained using our method DCQ, and (d) weights histogram of the convolutional layer highlighting the altered binary weights after training (using DCQ) relative to the original distribution}	
    \vspace{-0.1in} 
    \label{fig:kernel_vis}
\end{figure*}
\begin{figure*}
    \centering
    \includegraphics[width=0.9\textwidth]{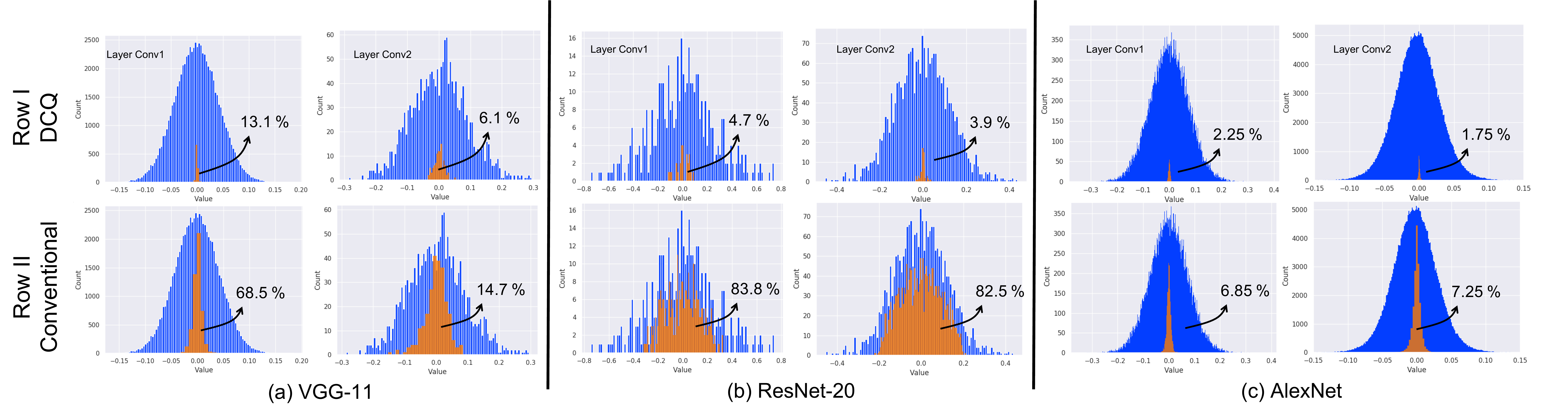}
    \caption{Weights histograms of the first two convolutional layers of three different DNNs: (a) VGG11, (b) ResNet-20, and (c) AlexNet, highlighting the altered portion of the trained binary weights (depicted percentages indicate the exact portion in orange) relative to the directly binarized weights. Original total weights histograms are shown in blue. Row I shows the results using our method (DCQ), and Row II shows for the conventional end-to-end taining method.}		
    \vspace{-0.2in}  
    \label{fig:kernel_hist}
\end{figure*}

\subsection{Analysis: DCQ vs Conventional Binary Kernels}
This section provides an analysis of our obtained binary weight kernels and sheds light on some interesting observations.
We start by posing the following questions: how are trained binary weight kernels different from just direct binarization from the original full precision weight kernels? and whether different training algorithms can yield qualitatively different binary weight kernels?

Figure \ref{fig:kernel_vis} shows a visualization of a subset of weight kernels from the second convolutional layer of LeNet and AlexNet.
(a) is the original full precision kernels, (b) direct binarization of full precision kernels, and (c) binarization after training (applying DCQ).
In the figure, weights that are different between the trained binary kernel and the directly binarized kernel are highlighted with square rectangles across the three visualizations.
Spatially contrasting those highlighted altered weights on the full precision kernels, it can be noticed that they mostly share a common feature that is being low in magnitude (shown as  white squares in (a)).
From statistical point of view, Figure \ref{fig:kernel_vis} (d) shows the original full precision weights histogram (in blue) and overlaying the portion of the altered weights (in light orange).
We can observe the following. First, during training, only very small percentage of the weights are actually altered relative to the total number of weights. Specifically, in this example, it is around 3.5\% and 2.25\% for LeNet and AlexNet respectively, of the total weights got impacted by training. 
Moreover, despite the marginal difference between the binary kernels, they experience dramatic accuracy difference: 10.8\% vs 98.1\% for kernels in (b) and (c) respectively, for LeNet, and 40.5\% and 55.6\% for AlexNet.

Now, to check whether this is a general trend and whether different training algorithms has an impact on this, we extend our statistical analysis to more networks.
Figure \ref{fig:kernel_hist} shows weight histograms of the first two convolutional layers of AlexNet, ResNet-20, and VGG-11.
As seen in the figure, first, for Figure \ref{fig:kernel_hist} Row I (DCQ), the altered portion of binary weights during training is consistently small in both number and magnitude across different layers and different networks.
Second, contrasting that behavior using DCQ vs using the conventional end-to-end quantized training, as shown in Figure \ref{fig:kernel_hist} Row II (Conventinoal), we see that binary weight kernels clearly encounter much more variations during the conventional end-to-end training as compared to our approach, \dcq.

Comparing the two training algorithms, \textit{\dcq yields minimal changes in the right place} to the binary weights as the entire technique is based on matching  the intermediate features represented by weight kernels.
%
%
Which, consequently, leads to faster convergence behavior and higher solution quality at the same time.
Moreover, this opens up the possibility of \textit{magnitude-constrained weight training} where only weights below a certain magnitude are set to be trainable which can potentially improve the optimization process further.
%

%
%
\subsection{Exploratory Studies}\label{sec:exploratory_studies}

%
%
%
\begin{figure*}[t]
    \centering
    \includegraphics[width=0.9\textwidth]{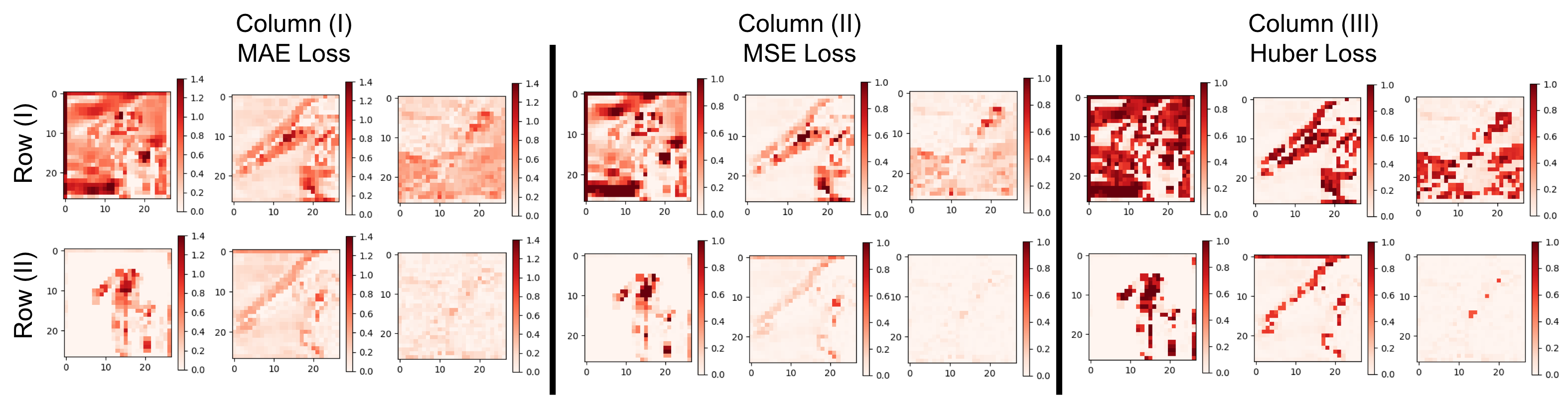}
    \caption{Loss visualization of intermediate feature maps samples. Row(I): before DCQ training, Row(II): after DCQ training. Columns show results for different loss formulations. Col(I) MAE, Col(II) MSE, and Col(III) Huber loss. The results are for the second convolution layer in AlexNet with binary quantization.}	
    \vspace{-3ex}
    \label{fig:loss}
\end{figure*}

\niparagraph{Impact of different loss formulations for intermediate learning.}
\begin{figure}
    \centering
    \includegraphics[width=0.45\textwidth]{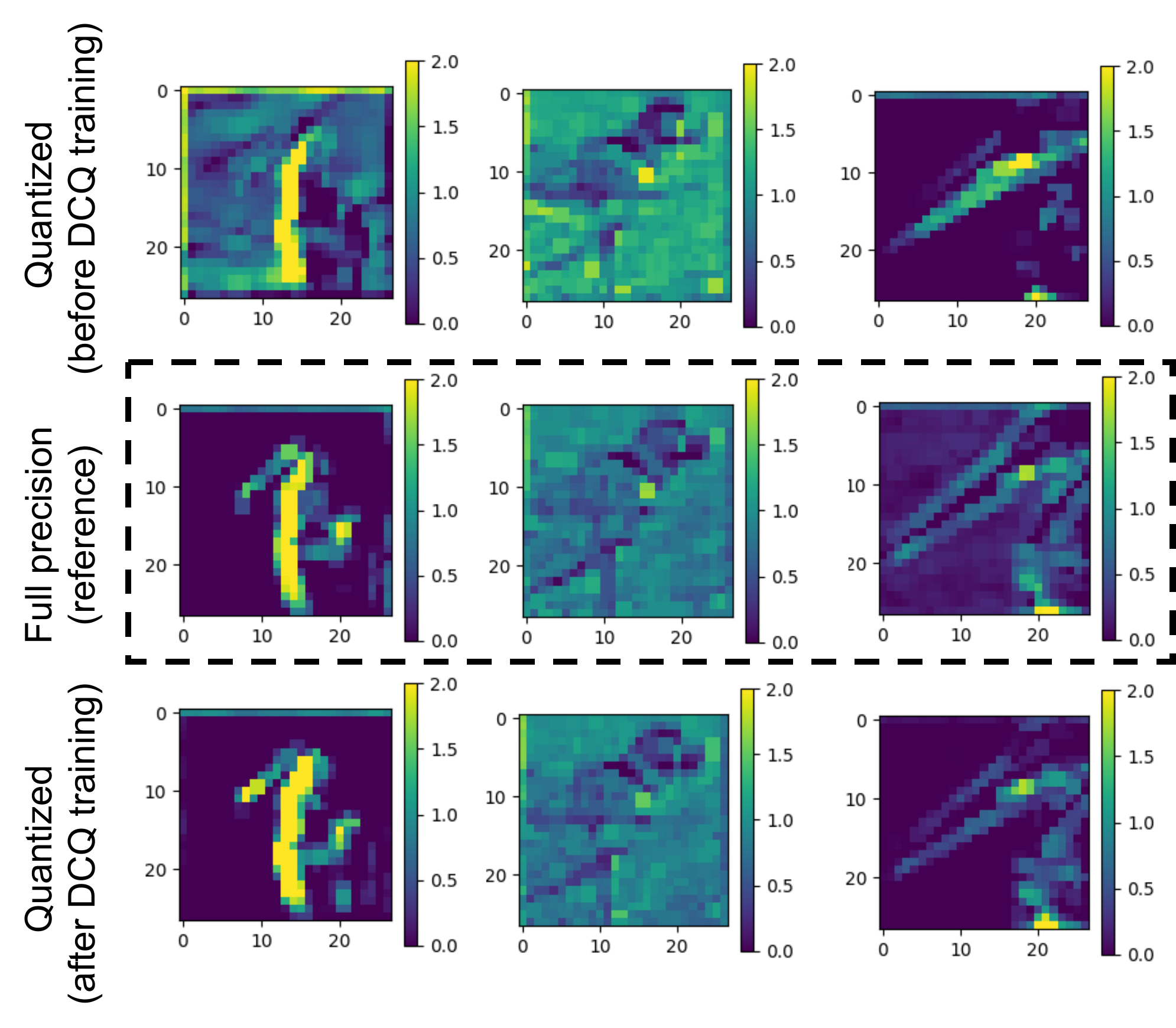}
    \caption{Feature maps before and after DCQ training compared to full precision maps. The results are for the second convolution layer in AlexNet with binary quantization.}
    \label{fig:fmaps}
\end{figure}
%
As mentioned in section~\ref{sec:loss_func}, we have examined three of the most commonly used loss formulations. Namely: (1) Mean Square Error (MSE); (2) Mean Absolute Error (MAE); (3) Huber Loss.
Figure~\ref{fig:loss} shows different samples of feature maps losses (for the second convolution layer of AlexNet wtih binary weights). Row(I) shows different samples of feature map losses before \dcq training. 
Row(II) shows the losses for the same samples after \dcq training (matching feature maps). Different columns show different loss formulations. Col(I): MSE Loss; Col(II): MAE Loss; and Col(III): Huber Loss.
As it can be seen, the feature map losses (the amount of redness) significantly decreases after \dcq training as a result of regressing the quantized model intermediate feature maps to the full precision counterparts.
We can also notice that the behavior is consistent across different regression losses. Nevertheless, based on our experimentation, among the considered formulations, MSE seems to be the most effective during the intermediate learning process. The trends are similar for the other networks.
Figure~\ref{fig:fmaps} compares visualizations of different samples of actual feature maps before and after \dcq training with respect to the full precision ones demonstrating the effectiveness of the proposed approach.
Lastly, divide and conquer is a very basic and universal engineering principle that is commonly and widely applied across a variety of fields. Here, we propose a procedure that extends such effective principle to quantized training of neural networks.
We also provide a preliminary analysis on the impact of the number of splitting points in the Appendix~\ref{appendix:splitting}.

\subsection{Impact of the number of splitting points.}\label{appendix:splitting}
\begin{figure}
    \centering
    \includegraphics[width=0.23\textwidth]{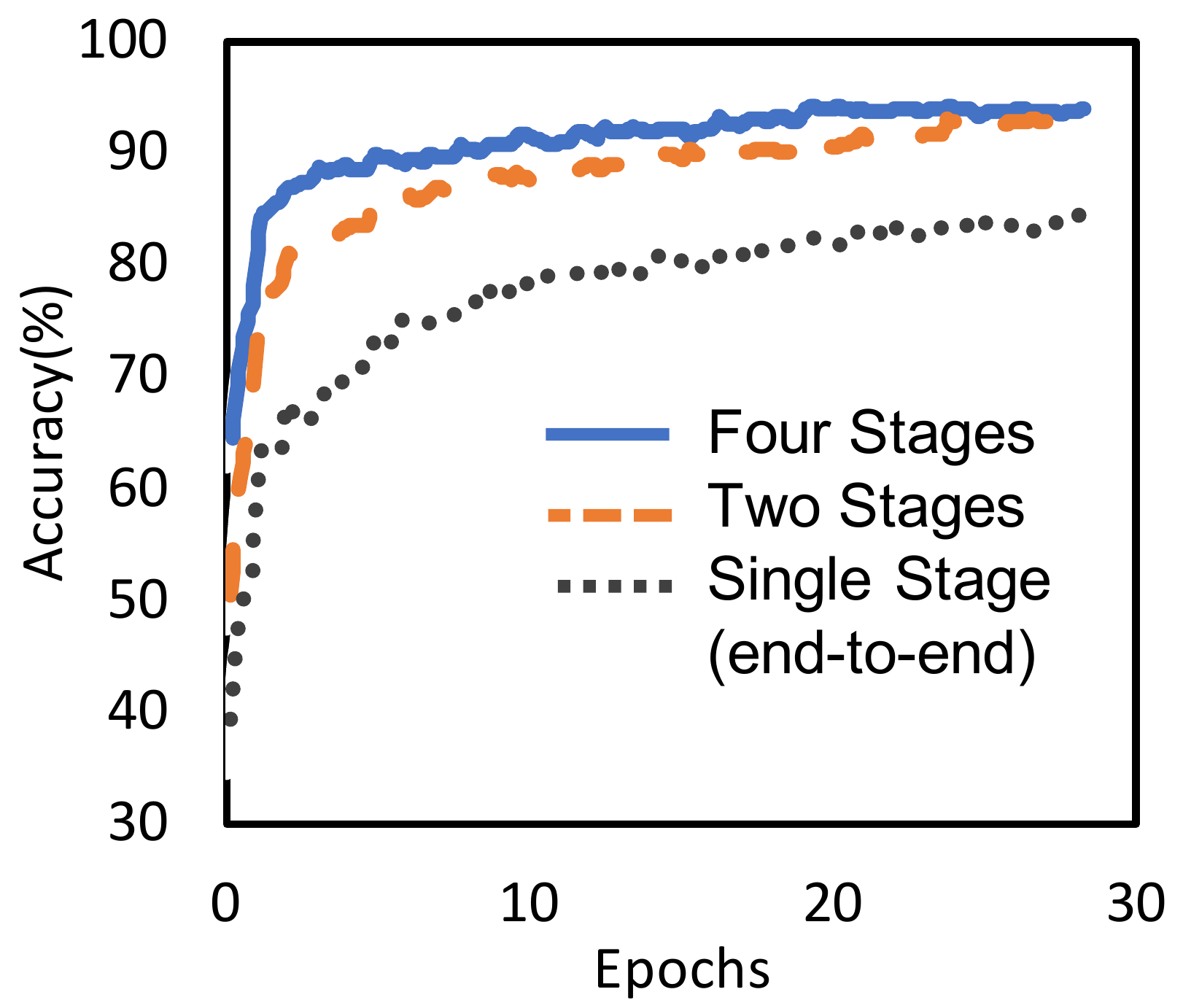}
    \caption{Impact of different splitting on the convergence behavior for VGG-11 (ternary quantization).}		  
    \vspace{-2ex}
    \label{fig:split}
\end{figure}
As number of splitting points increases, the large optimization problem gets divided into smaller subproblems.
Thus, on one side, it becomes easier to solve each subproblem separately.
On the other side, however, the complexity overhead increases as well.
We leave the optimal choice of how many stages a network should be divided and how many layers per stage to future work.
Here, we provide one experimental example to give some intuition about the impact of different splitting points.
Figure \ref{fig:split} shows the convergence behavior for different splittings of VGG-11: four-stage and two-stage splitting as compared to single stage (conventional knowledge distillation).
As seen in the figure, not only the convergence is faster as number of stages increases but also it eventually converges to a higher final accuracy as compared to lesser number of stages or no splitting at all.

%% file: body/theory_backup.tex
\section{Theoretical guarantees of DCQ: Upper bound on accumulated residual quantization error through splitting sections}
One issue that arises as a result of the strategy of splitting into sections and training each section separately is accumulation of error residuals through sections which may impact the overall performance of the proposed technique.
Here, we theoretically derive an upper bound on the total accumulated error across the resulting subnetworks after splitting using a chaining argument and utilizing Lipschitz continuity. 
\subsection{Error Upper Bound: Single-stage}\label{singlestage}
We control the neural network’s Lipschitz constant \cite{szegedy2013intriguing,bartlett2017spectrally,cisse2017parseval,gouk2018regularisation} to suppress network’s accumulation of error.
The Lipschitz constant describes: when input changes, how much does the output change correspondingly. 
For a function $f:X\rightarrow Y$, if it satisfies
\[ \|f(x_1)- f(x_2)\|_Y \leq L\|x_1- x_2\|_X, \ \ \forall \ x_1, x_2 \in X \]
for $L\geq 0$, and norms $\|\cdot\|_X$ and $\|\cdot\|_Y$ on their respective spaces, then we call $h$ Lipschitz continuous and  $L$ is the known as the Lipschitz constant of $h$.

Let us consider a sub-network (section) $i$ with full precision weights $W_{fp}$ and corresponding quantized inputs $W_q$, and let the full precision network output be $f_{fp}(x) = \sigma(W_{fp}x)$ for some activation function $\sigma$ and the quantized network output be $f_q(x) = \sigma(W_q x)$.
For a one layer network, full precision network $f_{fp}$ has Lipschitz constant $L$, which satisfies 
$$ L\le C_\sigma \|W_{fp}\| \textnormal{ for } C_\sigma = \frac{d\sigma}{dx}.$$
This bound is immediate from the fact that $\nabla f_{fp}(x) = \sigma'(W_{fp}x)\cdot \begin{bmatrix}
W_{\cdot, 1} & ... & W_{\cdot, d}
\end{bmatrix}$, and $L \le \max_x \|\nabla f_{fp}(x)\|$.  

Assume the application of our quantization scheme leads to an error in the output of size $\|f_{fp}(x) - f_{q}(x)\|<\delta$. This comes from the quantization error guarantee of the used technique.  
Under this model, we can use a simple triangle inequality to get:
$$ \|f_q(x)-f_q(y)\| < L \|x-y\| + 2\delta.$$
The problem now comes down to chaining layers of a network together.  If we were ignoring quantization, a tighter bound for the Lipschitz constant of $f_{fp}$ can be found in \cite{virmaux2018lipschitz} for arbitrary networks and \cite{zou2019lipschitz} for particular convolutional networks.  However, these approaches are not condusive to analyzing the layer-wise quantization error.

\subsection{Error Upper Bound: N-stages (Network-wide)}\label{Nstage}
Let’s consider a feed-forward network with the following function formulation.
\[ f(x) = (\phi^{(m)} \circ \ \phi^{(m-1)} \circ  ... \circ \ \phi^{(1)}) (x) \]
%
where $\phi^{(i)}$ is a given layer of the network with Lipschitz constant $L_i$. 

\begin{lemma}\label{lipschitz}
Let $f_{fp}$ be an $m$ layer network, and each layer has Lipschitz constant $L_i$.  Assume that quantizing each layer leads to a maximum pointwise error of $\delta_i$, and results in a quantized $m$ layer network $f_q$.  Then for any two points $x,y\in X$, $f_q$ satisfies
\begin{align*}
\|f_q(x) - f_q(y)\| <  \left(\prod_{j=1}^m L_j \right) \|x-y\| + 2 \Delta_{m,L},
\end{align*}
where $\Delta_{m,L} = \delta_m + \sum_{i=1}^{m-1} \left(\prod_{j=i+1}^m L_j\right) \delta_i$.
\end{lemma}

\begin{proof}
Let $\phi^{(i)}_q$ be the quantized $i^{th}$ layer of the network.
From Section \ref{singlestage}, we know that 
$$\|\phi^{(i)}_q(x)-\phi^{(i)}_q(y)\| < L_i\|x-y\| + 2\delta_i.$$
Similarly, we know that feeding in the previous layer's quantized output yields
\begin{align*}
\|\phi^{(2)}_q\circ \phi^{(1)}_q(x)-\phi^{(2)}_q\circ \phi^{(1)}_q(y)\| &\le  L_2\| \phi^{(1)}_q(x)-\phi^{(1)}_q(y)\| + 2\delta_2\\
&\le L_2 L_1 \|x-y\| + 2 L_2 \delta_1 + 2\delta_2.
\end{align*}
By chaining together the $i$ layers inductively up to $m$, 
we complete the desired inequality.
\end{proof}

A common practice to keep the product of Lipschitz constants small is by keeping the constants small using regularization or weight clipping. So if we quantize a DNN with a small constant, then the error across the entire network is close to just the sum of quantization errors. Experimentally, Lipschitz constant of each layer is found empirically by taking $\max_{x,y} \|\phi_i(x)-\phi_i(y)\| / \|\phi_{i-1}(x)-\phi_{i-1}(y)\|$. 

We also bound the error between the quantized network and the full precision network using similar arguments.
\begin{lemma}
Under the same assumptions as Lemma \ref{lipschitz}, for a point $x\in X$, $f_q$ satisfies
\begin{align*}
    \|f_q(x) - f_{fp}(x)\| \le 3\Delta_{m,L}.
\end{align*}
\end{lemma}
\begin{proof}
We know that $\|\phi^{(1)}_q(x) - \phi^{(1)}(x)\|<\delta_1$.  This means $\phi^{(2)}$ receives different inputs depending on whether $\phi^{(1)}$ was quantized or not, and thus requires the Lipschitz bound.  Thus

\begin{dmath*}
    \|\phi^{(2)}_q((\phi^{(1)}_q(x)) - \phi^{(2)}(\phi^{(1)}(x))\| 
    \le \|\phi^{(2)}_q(\phi^{(1)}_q(x)) - \phi^{(2)}_q(\phi^{(1)}(x))\| 
\\    + \|\phi^{(2)}_q(\phi^{(1)}(x)) - \phi^{(2)}(\phi^{(1)}(x))\|  
    \le L_2 \|\phi^{(1)}_q(x) - \phi^{(1)}(x)\| + \\ 2\delta_2
    + \delta_2  \le  2L_2 \delta_1 + 3\delta_2.
\end{dmath*}
Chaining the argument for the $i^{th}$ layer inductively up to $m$, we arrive at the desired inequality.
\end{proof}

As the Lipschitz constant of the network is the product of its individual layers’ Lipschitz constants, $L$ can grow exponentially if $L_i \geq 1$.
This is the common case for normal network training \cite{cisse2017parseval}, and thus the perturbation will be amplified for such a network. 
Therefore, to keep the Lipschitz constant of the whole network small, we need to keep the Lipschitz constant of each layer $L_i < 1$. 
We call a network with $L_i < 1, \forall i = 1, ..., \ L$ a non-expansive network.
%

\subsection{Lipschitz Constants in Classification Networks}

The Lipschitz constant is traditionally defined for regression problems where $f$ can take arbitrary values on $\mathbb{R}$, but it also has implications for classification networks.  For a classification network, the input is labeled data $(x_i, y_i)$ for $y_i$ coming from one of $K$ classes.
Then the output $f(x)$ is a function $f:X\rightarrow \mathbb{R}^K$.  Assume $(x_i, y_i)$ satisfies $y_i = k$ and $f$ can perfectly classify $x_i$.  Then if $f$ has a softmax output, the value $f(x_i) = e_k$ for $e_k$ being the unit vector on the $k^{th}$ class.  
A common problem for classification networks is to determine how much one can perturb the data point $x_i$ and maintain the correct classification (i.e., $f(x_i+\eta)_k > 1/2$).  For a more general network that isn't softmax, we can define the distance between classifications as $r = \frac{1}{2}\min_{\{i,j : y_i \neq y_j\}} \|y_i - y_j\|_2$.  This reduces to $r=1/2$ for softmax.  And this leads to the following theorem.

\begin{theorem}
Let $r= \frac{1}{2}\min_{\{i,j : y_i \neq y_j\}} \|y_i - y_j\|_2$, and let $f_{fp}$ and $f_q$ be the full precision and quantized $m$ layer networks as in Section \ref{Nstage}.  Let $L = \prod_{i=1}^m L_i$ be the Lipschitz constant of $f_{fp}$.  Then for any $(x_i,y_i)$ that $f_{q}$ perfectly classifies and any perturbation $\eta$ such that $\|\eta\| < \frac{r - 2\Delta_{m,L}}{L}$, $f_q$ will also classify $x_i + \eta$ with label $y_i$.

Moreover, if $f_{fp}$ perfectly classifies $x_i$, then for any perturbation $\eta$ such that $\|\eta\| < \frac{r-5\Delta_{m,L}}{L}$, $f_q$ will also classify $x_i + \eta$ with label $y_i$. 
\end{theorem}

\begin{proof}
From the guarantee of Lemma \ref{lipschitz}, we know $$\|f_q(x+\eta) - f_q(x)\|\le L\|(x+\eta) - x\| + 2\Delta_{m,L}$$.  Under the assumption on the norm of $\eta$, we get
\begin{align*}
    \|f_q(x+\eta) - f_q(x)\|\le r.
\end{align*}
Thus given a perfect classification $f_q(x_i)=e_k$ for $y_i=k$, $f_q(x+\eta)$ will similarly classify $x_i+\eta$ as $y_i$.

If we consider a full precision network $f_{fp}$ that classifies $x_i$ perfectly, then we must simply apply a triangle inequality to attain
\begin{align*}
	\|f_q(x+\eta) - f_{fp}(x)\| &\le \|f_q(x+\eta) - f_q(x)\| + \|f_q(x) - f_{fp}(x)\| \\
	&\le L\|(x+\eta) - x\| + 2\Delta_{m,L}  + 3\Delta_{m,L}.
\end{align*}
Thus for $\eta$ such that $\|\eta\|< \frac{r-5\Delta_{m,L}}{L}$, we will attain $\|f_q(x+\eta) - f_{fp}(x)\|$ and classify $x+\eta$ correctly by the same argument.
\end{proof}


%% file: body/related.tex
\section{Related Work}

\niparagraph{Knowledge distillation.}
Knowledge distillation~\cite{DBLP:journals/corr/HintonVD15} is proposed to attain a smaller/shallower neural network (student) from one or an ensemble of bigger deep networks (teacher).
The student network is trained on a softened version of the \emph{final} output of teacher(s)~\cite{DBLP:conf/kdd/BucilaCN06}.
\textsc{FitNets}~\cite{DBLP:journals/corr/RomeroBKCGB14} extends knowledge distillation by extracting a hint from the teacher to train even a deeper but thinner student.
The hint is an intermediate feature representation of the teacher, that is used as a regularizer to pretrain the first few layers of the deep and thin student network.
%
%
After the pretraining phase, the full knowledge distillation is used to finish the training of the student.
\textsc{FitNets}~\cite{DBLP:journals/corr/RomeroBKCGB14} does not explore hints from more than one intermediate layer of the teacher.
Furthermore, \textsc{FitNets} applies the knowledge distillation pass over the entire student network at once.
\textsc{FitNets} are a complementary approach to our sectional knowledge distillation and similar hints can be utilized for each section.
Nonetheless, the following discusses the differences.
In contrast to this technique, \dcq (1) partitions the neural network to multiple independent sections and (2) applies knowledge distillation to each section in isolation and trains them independently, (3) not utilizing the intermediate representations as hint for pretraining.
(4) After the sections are trained through knowledge distillation, they are put together instead of applying another phase of training as done in \textsc{FitNets}~\cite{DBLP:journals/corr/RomeroBKCGB14}.
%
%
(5) Moreover, \dcq, exclusively, applies various regression losses in matching the quantized network intermediate feature maps to the corresponding full precision ones in a stage wise fashion. 
(6) Last but not least, the objective differ as the knowledge distillation and \textsc{FitNets} aim to compress the network while \dcq quantizes it preserving the teacher's original network architecture.

Other work~\cite{DBLP:conf/cvpr/YimJBK17} proposes an information metric, in terms of inter-layer flow (the inner product of feature maps), using which a teacher DNN can transfer the distilled knowledge to other student DNNs.

Knowledge distillation is also used for training a lower bitwidth student network from a full-precision teacher~\cite{apprentice:2017,DBLP:conf/iclr/PolinoPA18,DBLP:conf/aaai/WangBSZCY19}.
However, these works do not partition the network as \dcq does and also do not utilize teacher's intermediate layers. 

%
%
%

\niparagraph{Other quantization techniques.}
Multiple techniques~\cite{Zhou2016DoReFaNetTL,Mishra2017WRPNWR,Zhu2016TrainedTQ} have been proposed for low bidwidth/quantized training of neural networks.
DoReFa-Net~\cite{Zhou2016DoReFaNetTL} uses straight through estimator~\cite{DBLP:journals/corr/BengioLC13} for quantization and extends it for any arbitrary $k$ bit quantization.
DoReFa-Net also proposes a method to train a CNNs with low bitwidth weights and activations, low bitwidth parameter gradients using deterministic quantization of weights, activations and stochastic quantization of activations.
%
%
TTQ~\cite{Zhu2016TrainedTQ} proposes a method to reduce the weights to ternary values by adding scaling coefficients to each layer.
These scaling coefficients are learnt during training and during deployment, weights are directly quantized to ternary bitwidths and these scaling coefficients are used to scale the weights during inference.
%
%
PACT~\cite{Choi2018PACTPC} proposes a technique for quantizing activations using an activation clipping parameter which is optimized during training.
%
There have also been a lot of efforts~\cite{Rastegari2016XNORNetIC,Li2016TernaryWN,Hubara2017QNN} to binarize neural networks at the cost of some accuracy loss.
%
%
%
%
%

However, these inspiring efforts do not introduce sectioning nor they leverage knowledge distillation in the context of either quantization or binarizing the neural networks.
%
%

%% file: body/conclusion.tex
\section{Conclusion}
\label{sec:conclusion}

Quantization offers a promising path forward to reduce the compute complexity and memory footprint of deep neural networks. 
This paper sets out to tackle the main challenge in quantization, recovering as much accuracy as possible.
To that end, we developed a sectional multi-backpropagation algorithm that leverages multiple instances of knowledge distillation and intermediate feature representations to teach a quantized student through divide and conquer.
This algorithm, \dcq, achieves significantly higher accuracy compared to the state-of-the-art quantization methods by exploring a new sectional approach towards knowledge distillation.

%% file: paper.bbl
\begin{thebibliography}{34}
\providecommand{\natexlab}[1]{#1}
\providecommand{\url}[1]{\texttt{#1}}
\expandafter\ifx\csname urlstyle\endcsname\relax
  \providecommand{\doi}[1]{doi: #1}\else
  \providecommand{\doi}{doi: \begingroup \urlstyle{rm}\Url}\fi

\bibitem[Bartlett et~al.(2017)Bartlett, Foster, and
  Telgarsky]{bartlett2017spectrally}
Bartlett, P.~L., Foster, D.~J., and Telgarsky, M.~J.
\newblock Spectrally-normalized margin bounds for neural networks.
\newblock In \emph{Advances in Neural Information Processing Systems}, pp.\
  6240--6249, 2017.

\bibitem[Bengio et~al.(2013)Bengio, L{\'{e}}onard, and
  Courville]{DBLP:journals/corr/BengioLC13}
Bengio, Y., L{\'{e}}onard, N., and Courville, A.~C.
\newblock Estimating or propagating gradients through stochastic neurons for
  conditional computation.
\newblock \emph{CoRR}, abs/1308.3432, 2013.
\newblock URL \url{http://arxiv.org/abs/1308.3432}.

\bibitem[Bucila et~al.(2006)Bucila, Caruana, and
  Niculescu{-}Mizil]{DBLP:conf/kdd/BucilaCN06}
Bucila, C., Caruana, R., and Niculescu{-}Mizil, A.
\newblock Model compression.
\newblock In Eliassi{-}Rad, T., Ungar, L.~H., Craven, M., and Gunopulos, D.
  (eds.), \emph{Proceedings of the Twelfth {ACM} {SIGKDD} International
  Conference on Knowledge Discovery and Data Mining, Philadelphia, PA, USA,
  August 20-23, 2006}, pp.\  535--541. {ACM}, 2006.
\newblock ISBN 1-59593-339-5.
\newblock \doi{10.1145/1150402.1150464}.
\newblock URL \url{https://doi.org/10.1145/1150402.1150464}.

\bibitem[Choi et~al.(2018)Choi, Wang, Venkataramani, Chuang, Srinivasan, and
  Gopalakrishnan]{Choi2018PACTPC}
Choi, J., Wang, Z., Venkataramani, S., Chuang, P. I.-J., Srinivasan, V., and
  Gopalakrishnan, K.
\newblock Pact: Parameterized clipping activation for quantized neural
  networks.
\newblock \emph{CoRR}, abs/1805.06085, 2018.

\bibitem[Cisse et~al.(2017)Cisse, Bojanowski, Grave, Dauphin, and
  Usunier]{cisse2017parseval}
Cisse, M., Bojanowski, P., Grave, E., Dauphin, Y., and Usunier, N.
\newblock Parseval networks: Improving robustness to adversarial examples.
\newblock In \emph{Proceedings of the 34th International Conference on Machine
  Learning-Volume 70}, pp.\  854--863. JMLR. org, 2017.

\bibitem[Courbariaux et~al.(2015)Courbariaux, Bengio, and
  David]{DBLP:conf/nips/CourbariauxBD15}
Courbariaux, M., Bengio, Y., and David, J.
\newblock Binaryconnect: Training deep neural networks with binary weights
  during propagations.
\newblock In Cortes, C., Lawrence, N.~D., Lee, D.~D., Sugiyama, M., and
  Garnett, R. (eds.), \emph{Advances in Neural Information Processing Systems
  28: Annual Conference on Neural Information Processing Systems 2015, December
  7-12, 2015, Montreal, Quebec, Canada}, pp.\  3123--3131, 2015.

\bibitem[Gong et~al.(2019)Gong, Liu, Jiang, Li, Hu, Lin, Yu, and
  Yan]{DBLP:journals/corr/abs-1908-05033}
Gong, R., Liu, X., Jiang, S., Li, T., Hu, P., Lin, J., Yu, F., and Yan, J.
\newblock Differentiable soft quantization: Bridging full-precision and low-bit
  neural networks.
\newblock \emph{CoRR}, abs/1908.05033, 2019.
\newblock URL \url{http://arxiv.org/abs/1908.05033}.

\bibitem[Gouk et~al.(2018)Gouk, Frank, Pfahringer, and
  Cree]{gouk2018regularisation}
Gouk, H., Frank, E., Pfahringer, B., and Cree, M.
\newblock Regularisation of neural networks by enforcing lipschitz continuity.
\newblock \emph{arXiv preprint arXiv:1804.04368}, 2018.

\bibitem[Gupta et~al.(2015)Gupta, Agrawal, Gopalakrishnan, and
  Narayanan]{DBLP:conf/icml/GuptaAGN15}
Gupta, S., Agrawal, A., Gopalakrishnan, K., and Narayanan, P.
\newblock Deep learning with limited numerical precision.
\newblock In Bach, F.~R. and Blei, D.~M. (eds.), \emph{Proceedings of the 32nd
  International Conference on Machine Learning, {ICML} 2015, Lille, France,
  6-11 July 2015}, volume~37 of \emph{{JMLR} Workshop and Conference
  Proceedings}, pp.\  1737--1746. JMLR.org, 2015.
\newblock URL \url{http://jmlr.org/proceedings/papers/v37/gupta15.html}.

\bibitem[Hauswald et~al.(2015)Hauswald, Laurenzano, Zhang, Li, Rovinski,
  Khurana, Dreslinski, Mudge, Petrucci, Tang, and Mars]{sirius}
Hauswald, J., Laurenzano, M., Zhang, Y., Li, C., Rovinski, A., Khurana, A.,
  Dreslinski, R.~G., Mudge, T.~N., Petrucci, V., Tang, L., and Mars, J.
\newblock Sirius: An open end-to-end voice and vision personal assistant and
  its implications for future warehouse scale computers.
\newblock In \emph{ASPLOS}, 2015.

\bibitem[Hinton et~al.(2015)Hinton, Vinyals, and
  Dean]{DBLP:journals/corr/HintonVD15}
Hinton, G.~E., Vinyals, O., and Dean, J.
\newblock Distilling the knowledge in a neural network.
\newblock \emph{CoRR}, abs/1503.02531, 2015.
\newblock URL \url{http://arxiv.org/abs/1503.02531}.

\bibitem[Hubara et~al.(2017{\natexlab{a}})Hubara, Courbariaux, Soudry,
  El{-}Yaniv, and Bengio]{DBLP:journals/jmlr/HubaraCSEB17}
Hubara, I., Courbariaux, M., Soudry, D., El{-}Yaniv, R., and Bengio, Y.
\newblock Quantized neural networks: Training neural networks with low
  precision weights and activations.
\newblock \emph{Journal of Machine Learning Research}, 18:\penalty0
  187:1--187:30, 2017{\natexlab{a}}.
\newblock URL \url{http://jmlr.org/papers/v18/16-456.html}.

\bibitem[Hubara et~al.(2017{\natexlab{b}})Hubara, Courbariaux, Soudry,
  El-Yaniv, and Bengio]{Hubara2017QNN}
Hubara, I., Courbariaux, M., Soudry, D., El-Yaniv, R., and Bengio, Y.
\newblock {Quantized Neural Networks: Training Neural Networks with Low
  Precision Weights and Activations}.
\newblock \emph{J. Mach. Learn. Res.}, 2017{\natexlab{b}}.

\bibitem[Jaderberg et~al.(2017)Jaderberg, Czarnecki, Osindero, Vinyals, Graves,
  Silver, and Kavukcuoglu]{DBLP:conf/icml/JaderbergCOVGSK17}
Jaderberg, M., Czarnecki, W.~M., Osindero, S., Vinyals, O., Graves, A., Silver,
  D., and Kavukcuoglu, K.
\newblock Decoupled neural interfaces using synthetic gradients.
\newblock In Precup, D. and Teh, Y.~W. (eds.), \emph{Proceedings of the 34th
  International Conference on Machine Learning, {ICML} 2017, Sydney, NSW,
  Australia, 6-11 August 2017}, volume~70 of \emph{Proceedings of Machine
  Learning Research}, pp.\  1627--1635. {PMLR}, 2017.
\newblock URL \url{http://proceedings.mlr.press/v70/jaderberg17a.html}.

\bibitem[Krizhevsky et~al.(2012)Krizhevsky, Sutskever, and
  Hinton]{DBLP:conf/nips/KrizhevskySH12}
Krizhevsky, A., Sutskever, I., and Hinton, G.~E.
\newblock Imagenet classification with deep convolutional neural networks.
\newblock In Bartlett, P.~L., Pereira, F. C.~N., Burges, C. J.~C., Bottou, L.,
  and Weinberger, K.~Q. (eds.), \emph{Advances in Neural Information Processing
  Systems 25: 26th Annual Conference on Neural Information Processing Systems
  2012. Proceedings of a meeting held December 3-6, 2012, Lake Tahoe, Nevada,
  United States.}, pp.\  1106--1114, 2012.

\bibitem[LeCun et~al.(1989)LeCun, Boser, Denker, Henderson, Howard, Hubbard,
  and Jackel]{bpzip}
LeCun, Y., Boser, B.~E., Denker, J.~S., Henderson, D., Howard, R.~E., Hubbard,
  W.~E., and Jackel, L.~D.
\newblock Backpropagation applied to handwritten zip code recognition.
\newblock \emph{Neural Computation}, 1:\penalty0 541--551, 1989.

\bibitem[LeCun et~al.(2015)LeCun, Bengio, and
  Hinton]{DBLP:journals/nature/LeCunBH15}
LeCun, Y., Bengio, Y., and Hinton, G.~E.
\newblock Deep learning.
\newblock \emph{Nature}, 521\penalty0 (7553):\penalty0 436--444, 2015.
\newblock \doi{10.1038/nature14539}.
\newblock URL \url{https://doi.org/10.1038/nature14539}.

\bibitem[Li \& Liu(2016)Li and Liu]{Li2016TernaryWN}
Li, F. and Liu, B.
\newblock {Ternary Weight Networks}.
\newblock \emph{CoRR}, abs/1605.04711, 2016.

\bibitem[Mishra \& Marr(2018)Mishra and Marr]{apprentice:2017}
Mishra, A. and Marr, D.
\newblock {Apprentice: Using Knowledge Distillation Techniques To Improve
  Low-Precision Network Accuracy}.
\newblock In \emph{International Conference on Learning Representations}, 2018.

\bibitem[Mishra et~al.(2018)Mishra, Nurvitadhi, Cook, and
  Marr]{Mishra2017WRPNWR}
Mishra, A.~K., Nurvitadhi, E., Cook, J.~J., and Marr, D.
\newblock {WRPN: Wide Reduced-Precision Networks}.
\newblock In \emph{ICLR}, 2018.

\bibitem[Polino et~al.(2018)Polino, Pascanu, and
  Alistarh]{DBLP:conf/iclr/PolinoPA18}
Polino, A., Pascanu, R., and Alistarh, D.
\newblock Model compression via distillation and quantization.
\newblock In \emph{6th International Conference on Learning Representations,
  {ICLR} 2018, Vancouver, BC, Canada, April 30 - May 3, 2018, Conference Track
  Proceedings}. OpenReview.net, 2018.
\newblock URL \url{https://openreview.net/forum?id=S1XolQbRW}.

\bibitem[Rastegari et~al.(2016{\natexlab{a}})Rastegari, Ordonez, Redmon, and
  Farhadi]{Rastegari2016XNORNetIC}
Rastegari, M., Ordonez, V., Redmon, J., and Farhadi, A.
\newblock {XNOR-Net: ImageNet Classification Using Binary Convolutional Neural
  Networks}.
\newblock In \emph{ECCV}, 2016{\natexlab{a}}.

\bibitem[Rastegari et~al.(2016{\natexlab{b}})Rastegari, Ordonez, Redmon, and
  Farhadi]{xnornet:eccv:2016}
Rastegari, M., Ordonez, V., Redmon, J., and Farhadi, A.
\newblock {XNOR-Net: ImageNet Classification Using Binary Convolutional Neural
  Networks}.
\newblock In \emph{European Conference on Computer Vision}, pp.\  525--542,
  2016{\natexlab{b}}.

\bibitem[Romero et~al.(2015)Romero, Ballas, Kahou, Chassang, Gatta, and
  Bengio]{DBLP:journals/corr/RomeroBKCGB14}
Romero, A., Ballas, N., Kahou, S.~E., Chassang, A., Gatta, C., and Bengio, Y.
\newblock Fitnets: Hints for thin deep nets.
\newblock In Bengio, Y. and LeCun, Y. (eds.), \emph{3rd International
  Conference on Learning Representations, {ICLR} 2015, San Diego, CA, USA, May
  7-9, 2015, Conference Track Proceedings}, 2015.
\newblock URL \url{http://arxiv.org/abs/1412.6550}.

\bibitem[Szegedy et~al.(2013)Szegedy, Zaremba, Sutskever, Bruna, Erhan,
  Goodfellow, and Fergus]{szegedy2013intriguing}
Szegedy, C., Zaremba, W., Sutskever, I., Bruna, J., Erhan, D., Goodfellow, I.,
  and Fergus, R.
\newblock Intriguing properties of neural networks.
\newblock \emph{arXiv preprint arXiv:1312.6199}, 2013.

\bibitem[Virmaux \& Scaman(2018)Virmaux and Scaman]{virmaux2018lipschitz}
Virmaux, A. and Scaman, K.
\newblock Lipschitz regularity of deep neural networks: analysis and efficient
  estimation.
\newblock In \emph{Advances in Neural Information Processing Systems}, pp.\
  3835--3844, 2018.

\bibitem[Wang et~al.(2019)Wang, Bao, Sun, Zhu, Cao, and
  Yu]{DBLP:conf/aaai/WangBSZCY19}
Wang, J., Bao, W., Sun, L., Zhu, X., Cao, B., and Yu, P.~S.
\newblock Private model compression via knowledge distillation.
\newblock In \emph{The Thirty-Third {AAAI} Conference on Artificial
  Intelligence, {AAAI} 2019, The Thirty-First Innovative Applications of
  Artificial Intelligence Conference, {IAAI} 2019, The Ninth {AAAI} Symposium
  on Educational Advances in Artificial Intelligence, {EAAI} 2019, Honolulu,
  Hawaii, USA, January 27 - February 1, 2019.}, pp.\  1190--1197. {AAAI} Press,
  2019.
\newblock ISBN 978-1-57735-809-1.
\newblock URL \url{https://aaai.org/ojs/index.php/AAAI/article/view/3913}.

\bibitem[Yim et~al.(2017)Yim, Joo, Bae, and Kim]{DBLP:conf/cvpr/YimJBK17}
Yim, J., Joo, D., Bae, J., and Kim, J.
\newblock A gift from knowledge distillation: Fast optimization, network
  minimization and transfer learning.
\newblock In \emph{2017 {IEEE} Conference on Computer Vision and Pattern
  Recognition, {CVPR} 2017, Honolulu, HI, USA, July 21-26, 2017}, pp.\
  7130--7138. {IEEE} Computer Society, 2017.
\newblock ISBN 978-1-5386-0457-1.
\newblock \doi{10.1109/CVPR.2017.754}.
\newblock URL \url{https://doi.org/10.1109/CVPR.2017.754}.

\bibitem[Zhang et~al.(2018)Zhang, Yang, Ye, and Hua]{DBLP:conf/eccv/ZhangYYH18}
Zhang, D., Yang, J., Ye, D., and Hua, G.
\newblock Lq-nets: Learned quantization for highly accurate and compact deep
  neural networks.
\newblock In Ferrari, V., Hebert, M., Sminchisescu, C., and Weiss, Y. (eds.),
  \emph{Computer Vision - {ECCV} 2018 - 15th European Conference, Munich,
  Germany, September 8-14, 2018, Proceedings, Part {VIII}}, volume 11212 of
  \emph{Lecture Notes in Computer Science}, pp.\  373--390. Springer, 2018.
\newblock ISBN 978-3-030-01236-6.
\newblock \doi{10.1007/978-3-030-01237-3\_23}.
\newblock URL \url{https://doi.org/10.1007/978-3-030-01237-3\_23}.

\bibitem[Zhou et~al.(2017)Zhou, Yao, Guo, Xu, and
  Chen]{DBLP:conf/iclr/ZhouYGXC17}
Zhou, A., Yao, A., Guo, Y., Xu, L., and Chen, Y.
\newblock Incremental network quantization: Towards lossless cnns with
  low-precision weights.
\newblock In \emph{5th International Conference on Learning Representations,
  {ICLR} 2017, Toulon, France, April 24-26, 2017, Conference Track
  Proceedings}. OpenReview.net, 2017.
\newblock URL \url{https://openreview.net/forum?id=HyQJ-mclg}.

\bibitem[Zhou et~al.(2016)Zhou, Ni, Zhou, Wen, Wu, and
  Zou]{Zhou2016DoReFaNetTL}
Zhou, S., Ni, Z., Zhou, X., Wen, H., Wu, Y., and Zou, Y.
\newblock Dorefa-net: Training low bitwidth convolutional neural networks with
  low bitwidth gradients.
\newblock \emph{CoRR}, abs/1606.06160, 2016.
\newblock URL \url{http://arxiv.org/abs/1606.06160}.

\bibitem[Zhu et~al.(2017)Zhu, Han, Mao, and Dally]{Zhu2016TrainedTQ}
Zhu, C., Han, S., Mao, H., and Dally, W.~J.
\newblock {Trained Ternary Quantization}.
\newblock In \emph{ICLR}, 2017.

\bibitem[Zmora et~al.(2018)Zmora, Jacob, and Novik]{neta_zmora_2018_1297430}
Zmora, N., Jacob, G., and Novik, G.
\newblock Neural network distiller, June 2018.
\newblock URL \url{https://doi.org/10.5281/zenodo.1297430}.

\bibitem[Zou et~al.(2019)Zou, Balan, and Singh]{zou2019lipschitz}
Zou, D., Balan, R., and Singh, M.
\newblock On lipschitz bounds of general convolutional neural networks.
\newblock \emph{IEEE Transactions on Information Theory}, 2019.

\end{thebibliography}
